\documentclass[reqno, 11pt]{article}
\usepackage{authblk}
\usepackage{graphics}
\usepackage{graphicx}
\usepackage{hyperref}
\usepackage{color}
\usepackage{url}
\usepackage{enumerate}
\usepackage{amsmath,amsthm,amssymb}
\usepackage{caption}
\usepackage{subcaption}
\usepackage{dsfont}
\usepackage{algpseudocode}
\usepackage{algorithm}
\usepackage{lscape}
\usepackage{mathtools}

\usepackage{lscape}

\newtheorem{theorem}{Theorem}[section]
\newtheorem{lemma}[theorem]{Lemma}
\newtheorem{theoremInf}[theorem]{Theorem (informal version)}
\newtheorem{corollary}[theorem]{Corollary}
\newtheorem{definition}[theorem]{Definition}
\newtheorem{proposition}[theorem]{Proposition}
\newtheorem{observe}[theorem]{Observation}

\newtheorem{remark1}[theorem]{Remark}

\newenvironment{remark}{\begin{remark1} \rm}{\end{remark1}}

\DeclareMathOperator\rect{rect}
\DeclareMathOperator\supp{supp}

\DeclarePairedDelimiter{\ceil}{\lceil}{\rceil}

\title{Provable approximation properties for deep neural networks}
\author[1]{Uri Shaham}
\author[2]{Alexander Cloninger}
\author[2]{Ronald R. Coifman}
\affil[1]{Statistics department, Yale University}
\affil[2]{Applied Mathematics program, Yale University}

\date{}                                           

\begin{document}
\maketitle

\begin{abstract}
We discuss approximation of functions using deep neural nets. Given a function $f$ on a $d$-dimensional manifold $\Gamma \subset \mathbb{R}^m$, we construct a sparsely-connected depth-4 neural network and bound its error in approximating $f$. The size of the network depends on dimension and curvature of the manifold $\Gamma$, the complexity of $f$, in terms of its wavelet description, and only weakly on the ambient dimension $m$. Essentially, our network computes wavelet functions, which are computed from Rectified Linear Units (ReLU). 

\end{abstract}

\section{Introduction} \label{sec:intro}
In the last decade, deep learning algorithms achieved unprecedented success and state-of-the-art results in various machine learning and artificial intelligence tasks, most notably image recognition, speech recognition, text analysis and Natural Language Processing~\cite{lecun2015deep}. Deep Neural Networks (DNNs) are general in the sense of their mechanism for learning features of the data. Nevertheless, in numerous cases, results obtained with DNNs outperformed previous state-of-the-art methods, often requiring significant domain knowledge, manifested in hand-crafted features.

Despite the great success of DNNs in many practical applications, the theoretical framework of DNNs is still lacking; along with some decades-old well-known results, developing aspects of such theoretical framework are the focus of much recent academic attention. In particular, some interesting topics are (1) specification of the network topology (i.e., depth, layer sizes), given a target function, in order to obtain certain approximation properties, (2) estimating the amount of training data needed in order to generalize to test data with high accuracy, and also (3) development of training algorithms with performance guarantees.

\subsection{The contribution of this work}
In this manuscript we discuss the first topic. Specifically, we prove a formal version of the following result:
\begin{theoremInf}
Let $\Gamma \subset \mathbb{R}^m$ be a smooth $d$-dimensional manifold, $f \in L_2(\Gamma)$ and let $\delta>0$ be an approximation level.  Then there exists a depth-4 sparsely-connected neural network with $N$ units where $N=N(\delta, \Gamma, f, m)$, computing the function $f_N$  such that 
\begin{equation}
\|f-f_N\|_2^2 \le \delta.
\end{equation}
\end{theoremInf}
The number $N=N(\delta, \Gamma, f, m)$ depends on the complexity of $f$, in terms of its wavelet representation, the curvature and dimension of the manifold $\Gamma$ and only weakly on the ambient dimension $m$, thus taking advantage of the possibility that $d \ll m$, which seems to be realistic in many practical applications.
Moreover, we specify the exact topology of such network, and show how it depends on the curvature of $\Gamma$, the complexity of $f$, and the dimensions $d$, and $m$.
Lastly, for two classes of functions we also provide approximation error rates: $L_2$ error rate for functions with sparse wavelet expansion and point-wise error rate for functions in $C^2$:
\begin{itemize}
\item if $f$ has wavelet coefficients in $l_1$ then there exists a depth-4 network and a constant $c$ so that
\begin{equation}
\|f-f_N\|_2^2 \le \frac{c}{N}
\end{equation}
\item if $f \in C^2$ and has bounded Hessian, then there exists a depth-4 network so that
\begin{equation}
\|f - f_N\|_\infty = O\left(N^{-\frac{2}{d}} \right).
\end{equation}
\end{itemize}
 
\subsection {The structure of this manuscript}
 The structure of this manuscript is as follows: in Section \ref{sec:relatedWork} we review some of the fundamental theoretical results in neural network analysis, as well as some of the recent theoretical developments. In Section \ref{sec:preliminaries} we give quick technical review of the mathematical methods and results that are used in our construction. In Section \ref{sec:Main} we describe our main result, namely construction of deep neural nets for approximating functions on smooth manifolds. In Section \ref{sec:counting} we specify the size of the network needed to learn a function $f$, in view of the construction of the previous section.
Section \ref{sec:conclusions} concludes this manuscript.

\subsection{Notation}
$\Gamma$ denotes a $d$-dimensional manifold in $\mathbb{R}^m$. $\{(U_i, \phi_i) \}$ denotes an atlas for $\Gamma$. Tangent hyper-planes to $\Gamma$ are denoted by $H_i$. $f$ and variants of it stand for the function to be approximated.  $\varphi, \psi$ are scaling (aka "father") and wavelet (aka "mother") functions, respectively. The wavelet terms are indexed by scale $k$ and offset $b$.
The support of a function $f$ is denoted by $\supp(f)$.

\section{Related work}
\label{sec:relatedWork}
There is a huge body of theoretical work in neural network research. In this section, we review some classical theoretical results on neural network theory, and discuss several recent theoretical works.

A well known result, proved independently by Cybenko \cite{cybenko1989approximation}, Hornik \cite{hornik1991approximation} and others states that Artificial Neural Networks (ANNs) with a single hidden layer of sigmoidal functions can approximate arbitrary closely any compactly supported continuous function. This result is known as the ``Universal Approximation Property''. It does not relate, however, the number of hidden units and the approximation accuracy; moreover, the hidden layer might contain a very large number of units. 
Several works propose extensions of the universal approximation property (see, for example\cite{girosi1990networks, girosi1995regularization}, for a regularization perspective and also using radial basis activation functions, \cite{leshno1993multilayer} for all activation functions that achieve the universal approximation property).

The first work to discuss the approximation error rate  was done by Barron \cite {barron1993universal}, who showed that given a function $f:\mathbb{R}^m \rightarrow \mathbb{R}$ with bounded first moment of the magnitude of the Fourier transform
 \begin{equation}
 C_f=\int_{\mathbb{R}^m}|w||\tilde{f}(w)| < \infty \label{eq:barronReq}
 \end{equation}
 there exists a neural net with a single hidden layer of $N$ sigmoid units, so that the output $f_N$ of the network satisfies
 \begin{equation}
 \|f-f_N \|_2^2 \le \frac{c_f}{N},
 \end{equation}
where $c_f$ is proportional to $C_f$. We note that the requirement~\eqref{eq:barronReq} gets more restrictive when the ambient dimension $m$ is large, and that the constant $c_f$ might scale with $m$. The dependence  on $m$ is improved in \cite{mhaskar2004tractability}, \cite{kurkova2002comparison}. In particular, in \cite{mhaskar2004tractability} the constant is improved to be polynomial in $m$.
For $r$ times differentiable functions, Mahskar \cite{mhaskar1996neural} constructs a network with a single hidden layer of $N$ sigmoid units (with weights that do not depend on the target function) that achieves an approximation error rate 
\begin{equation}
\|f-f_N \|_2^2 = \frac{c}{N^{2r/m}},
\end{equation}
which is known to be optimal. This rate is also achieved (point-wise) in this manuscript, however, with respect to the dimension $d$ of the manifold, instead of $m$, which might be a significant difference when $d \ll m$. 

During the decade of $1990$s, a popular direction in neural network research was to construct neural networks in which the hidden units compute wavelets functions (see, for example \cite {zhang1992wavelet}, \cite{pati1993analysis} and \cite{zhao1998multidimensional}). These works, however, do not give any specification of network architecture to obtain desired approximation properties.

Several most interesting recent theoretical results consider the representation properties of neural nets. 
Eldan and Shamir~\cite{eldan2015power} construct a radial function that is efficiently expressible by a 3-layer net, while requiring exponentially many units to be represented accurately by shallower nets. 
In \cite{montufar2014number}, Montufar et al. show that DNNs can represent more complex functions than can represent a shallow network with the same number of units, where complexity is defined as the number of linear regions of the function. 
Tishby and Zaslavsky~\cite {tishby2015deep} propose to evaluate the representations obtained by deep networks via the information bottleneck principle, which is a trade-off between compression of the input representation and predictive ability of the output function, however do not provide any theoretical results.

A recent work by Chui and Mhaskar brought to our attention~\cite{Chui2015deep} constructs a network with similar functionality to the network we construct in this manuscript. In their network the low layers map the data to local coordinates on the manifold and the upper ones approximate a target function on each chart, however using B-splines.
 

\section{Preliminaries}
\label{sec:preliminaries}

\subsection {Compact manifolds in $\mathbb{R}^m$}
\label{sec:compactManifolds}
In this section we review the concepts of \textit{smooth manifolds}, \textit{atlases} and \textit{partition of unity}, which will all play important roles in our construction.

Let $\Gamma \subseteq \mathbb{R}^m$ be a compact $d$-dimensional manifold. We further assume that $\Gamma$ is smooth, and that there exists $\delta>0$ so that for all $x\in\Gamma$, $B(x,\delta)\cap \Gamma$ is diffeomorphic to a disc, with a map that is close to the identity.

\begin{definition}
A  \textbf{chart} for $\Gamma$ is a pair $(U, \phi) $ such that $U \subseteq \Gamma$ is open and 
\begin{equation}
\phi:U \rightarrow M,
\end{equation}
where $\phi$ is a homeomorphism and $M$ is an open subset of a Euclidean space.
\end{definition}
One way to think of a chart is as a tangent plane at some point $x \in U \subseteq \Gamma$, such that the plane defines a Euclidean coordinate system on $U$ via the map $\phi$.

\begin{definition}
An  \textbf{atlas} for $\Gamma$ is a collection $\{(U_i, \phi_i) \}_{i \in I}$ of charts such that $\cup_i U_i=\Gamma $.
\end{definition}
\begin{definition}
Let $\Gamma$ be a smooth manifold. A \textbf{partition of unity} of $\Gamma$ w.r.t an open cover $\{U_i \}_{i\in I}$ is a family of nonnegative smooth functions $\{\eta_i\}_{i\in I}$ such that for every $x \in \Gamma$, 
$\sum_i \eta_i(x)=1$ and for every $i$, $\supp(\eta_i) \subseteq(U_i)$.
\end{definition}

\begin{theorem} (Proposition $13.9$ in \cite{loring2008introduction}) 
\label{thm:PartitionOfUnity}
Let $\Gamma$ be a compact manifold and $\{U_i \}_{i \in I}$ be an open cover of $\Gamma$. Then there exists a partition of unity $\{\eta_i\}_{i\in I}$ such that for each $i$, $\eta_i$ is in $C^\infty$, has compact support and $\supp(\eta_i)\subseteq U_i$.

\end{theorem}


\subsection {Harmonic analysis on spaces of homogeneous type}
\label{sec:Harmonic}
\subsubsection {Construction of wavelet frames}
In this section we cite several standard results, mostly from \cite {deng2009harmonic}, showing how to construct a wavelet frame of $L_2(\mathbb{R}^d)$, and discuss some of its properties.

\begin{definition}(Definition $1.1$ in \cite {deng2009harmonic})\\
A \textbf{space of homogeneous type} $(\mathcal{X}, \mu, \delta)$ is a set $\mathcal{X}$ together with a measure $\mu$ and a quasi-metric $\delta$ (satisfies triangle inequality up to a constant $A$) such that for every $x\in \mathcal{X},\; r>0$
\begin{itemize}
\item $0<\mu(B(x,r))<\infty$
\item There exists a constant $A'$ such that $\mu(B(x,2r))\le A' \mu(B(x,r))$
\end{itemize}
\end{definition}
In this manuscript, we are interested in constructing a wavelet frame on $\mathbb{R}^d$, which, equipped with Lebesgue measure and the Euclidean metric, is a  space of homogeneous type.

\begin{definition}(Definition $3.14$ in \cite {deng2009harmonic})\label{def:fatherMother}\\
Let $(\mathcal{X}, \mu, \delta)$ be a space of homogeneous type. 
A family of functions $\{S_k \}_{k\in \mathbb{Z}}$, $S_k: \mathcal{X} \times \mathcal{X} \rightarrow \mathbb{C}$ is said to be a family of \textbf{averaging kernels} (``father functions'') if conditions  $3.14-3.18$ and $3.19$ with $\sigma=\epsilon$  in \cite {deng2009harmonic} are satisfied. A family $\{D_k \}_{k\in \mathbb{Z}}$, $D_k: \mathcal{X} \times \mathcal{X} \rightarrow \mathbb{C}$ is said to be a family of  (``mother'') \textbf{wavelets} if for all $x,y\in \mathcal{X}$, 
\begin{equation} \label{eq:motherD}
D_k(x,y) =  S_k(x,y)-S_{k-1}(x,y),
\end{equation}
and $S_k, S_{k-1}$ are averaging kernels.
\end{definition}
By standard wavelet terminology, we denote
\begin{equation}
\psi_{k,b}(x) \equiv 2^{-\frac{k}{2}} D_{k}(x,b). \label{eq:mother}
\end{equation}
\begin{theorem} (A simplified version of Theorem $3.25$ in \cite {deng2009harmonic}) \label{thm:3.25}\\
Let $\{S_k \}$ be a family of averaging kernels. Then there exist families $\{\psi_{k,b}\},\{\widetilde{\psi}_{k,b}\}$  
such that for all $f \in L_2(\mathbb{R}^d)$
\begin{equation}
f(x) = \sum_{(k,b)\in\Lambda}  \langle f,\widetilde{\psi}_{k,b}\rangle\psi_{k,b}(x)
\end{equation}
Where the functions $\psi_{k,b}$ are given by Equations \eqref{eq:motherD} and \eqref{eq:mother} and $\Lambda =\{(k,b) \in \mathbb{Z} \times \mathbb{R}^d$: $b \in 2^{-\frac{k}{d}}\mathbb{Z}^d\}$.
\end{theorem}

\begin{remark}
\label{remark:smoothness}
The kernels $\{S_{k}\}$ need to be such that for every $x \in \mathbb{R}^d$, $\sum_{(k,b)\in\Lambda} S_{k}(x,b)$ is sufficiently large. This is discussed in great generality in chapter 3 in \cite {deng2009harmonic}.
\end{remark}
 
\begin{remark}
\label{remark:duals}
The functions $\widetilde{\psi}_{k,b}$ are called dual elements, and are also a wavelet frame of $L_2(\mathbb{R}^d)$.
\end{remark}

%

\subsection{Approximation of functions with sparse wavelet coefficients}
\label{sec:barronSparse}
In this section we cite a result from \cite{barron2008approximation} regarding approximating functions which have sparse representation with respect to a dictionary $\mathcal{D}$ using finite linear combinations of dictionary elements.

Let $f$ a function in some Hilbert space $\mathcal{H}$ with inner product $\langle\cdot,\cdot \rangle$ and norm $\|\cdot \|$, and let $\mathcal{D} \subset \mathcal{H}$ be a dictionary, i.e., any family  of functions $(g)_{g\in \mathcal{D}}$ with unit norm. 
Assume that $f$ can be represented as a linear combination of elements in $\mathcal{D}$ with absolutely summable coefficients, and denote the sum of absolute values of the coefficients in the expansion of $f$ by $\|f\|_{\mathcal{L}_1}$.

In \cite{barron2008approximation}, it is shown that $\mathcal{L}_1$ functions can be approximated using $N$ dictionary terms with squared error proportional to $\frac{1}{\sqrt{N}}$. As a bonus, we also get a greedy algorithm (though not always practical) for selecting the corresponding dictionary terms. 
OGA is a greedy algorithm that at the $k$'th iteration computes the residual
\begin{equation}
r_{k-1} := f-f_{k-1},
\end{equation}
finds the dictionary element that is most correlated with it
\begin{equation}
g_k \in \arg\max_{g\in\mathcal{D}}|\langle r_{k-1},g \rangle |
\end{equation}
and defines a new approximation
\begin{equation}
f_k := P_kf,
\end{equation}
where $P_k$ is the orthogonal projection operator onto $\text{span}\{g_1,...,g_k\}$. 

\begin{theorem}(Theorem 2.1 from \cite{barron2008approximation})
\label{thm:BarronGreedy}
The error $r_N$ of the OGA satisfies
\begin{equation}
\|f-f_N \| \le \|f\|_{\mathcal{L}_1}(N+1)^{-1/2}.
\end{equation}
\end{theorem}

Clearly, for $\mathcal{H} = L_2(\mathbb{R}^d)$ we can choose the dictionary to be the wavelet frame given by  
\begin{equation}
\mathcal{D} = \{\psi_{k,b}: (k,b) \in \mathcal{Z} \times \mathbb{R}^d, b\in 2^{-k}\mathbb{Z} \}.
\end{equation}

\begin{remark} \label{remark:equivalentFrames}
Let $\mathcal{D} =\{\psi_{k,b} \}$ be a wavelet frame that satisfies the regularities in conditions $3.14-3.19$ in \cite{deng2009harmonic}. Then if a function $f$ is in $\mathcal{L}_1$ with respect to $\mathcal{D}$, it is also in  $\mathcal{L}_1$ with respect to any other wavelet frame that satisfies the same regularities. In other words, having expansion coefficients in $l_1$ does not depend on the specific choice of wavelets (as long as the regularities are satisfied). The idea behind the proof of this claim is explained in appendix~\ref{app:equivFrame}.  
\end{remark}

\begin{remark}
Section $4.5$ in \cite {deng2009harmonic} gives a way to check whether a function $f$ has sparse coefficients without actually calculating the coefficients: 
\begin{equation}
f \in \mathcal{L}_1 \text{ iff } \sum_{k \in \mathbb{Z}} 2^{k/2} \|f * \psi_{k,0}\|_1<\infty,
\end{equation}
i.e., one can determine if $f \in \mathcal{L}_1 $ without explicitly computing its wavelet coefficients; rather, by convolving $f$ with non-shifted wavelet terms in all scales.
\end{remark}


\section{Approximating functions on manifolds using deep neural nets}
\label{sec:Main}
In this section we describe in detail the steps in our construction of deep networks, which are designed to approximate functions on smooth manifolds. The main steps in our construction are the following:
\begin{enumerate}
\item We construct a frame of $L_2(\mathbb{R}^d)$ in which the frame elements can be constructed from rectified linear units (see Section \ref{sec:wavConstruction}).
 \item Given a $d$-dimensional manifold $\Gamma \subset \mathbb{R}^m$, we construct an atlas for $\Gamma$ by covering it with open balls (see Section \ref{sec:creatingAtlas}).
\item We use the open cover to obtain a partition of unity of $\Gamma$ and consequently represent any function on $\Gamma$ as a sum of functions on $\mathbb{R}^d$ (see section \ref{sec:sumOfFunctions}).
\item We show how to extend the wavelet terms in the wavelet expansion, which are defined on $\mathbb{R}^d$, to $\mathbb{R}^m$ in a way that depends on the curvature of the manifold $\Gamma$ (see Section \ref{sec:extension}). 
\end{enumerate}

\subsection{Constructing a wavelet frame from rectifier units}
\label{sec:wavConstruction}
In this section we show how Rectified Linear Units (ReLU) can be used to obtain a wavelet frame of $L_2(\mathbb{R}^d)$. The construction of wavelets from rectifiers is fairly simple, and we refer to results from Section~\ref{sec:Harmonic} to show that they obtain a frame of $L_2(\mathbb{R}^d)$.

The rectifier activation function is defined on $\mathbb{R}$ as
\begin{equation}
\rect(x) = \max\{0,x\}.
\end{equation}
we define a trapezoid-shaped function $t:\mathbb{R} \rightarrow\mathbb{R}$ by
\begin{equation}
t(x) = \rect(x+3) - \rect(x+1) - \rect(x-1) + \rect(x-3).
\end{equation}

We then define the scaling function $\varphi:\mathbb{R}^d \rightarrow \mathbb{R}$ by 
\begin{equation}
\varphi(x) = C_d\rect\left(\sum_{j=1}^d t(x_j) -2(d-1)\right), \label{eq:father}
\end{equation}
where the constant $C_d$ is such that  
\begin{equation}
\int_{\mathbb{R}^d} \varphi(x)dx=1;
\end{equation}
for example, $C_1=\frac{1}{8}$.
Following the construction in Section \ref{sec:Harmonic}, we define
\begin{equation}
S_k(x,b) = 2^k\varphi(2^{\frac{k}{d}}(x-b)) \label{eq:Sdef}
\end{equation}
\begin{lemma}
\label{lemma:conditions}
The family $\{S_k \}$ is a family of averaging kernels.
\end{lemma}
The proof is given in Appendix \ref{app:conditions}. Next we define the (``mother'') wavelet as
\begin{equation}
D_k(x,b) = S_k(x,b) - S_{k-1}(x,b),
\end{equation}
And denote
\begin{equation}
\psi_{k,b}(x) \equiv 2^{-\frac{k}{2}}D_k(x,b),
\end{equation} 
and
\begin{align}
\psi(x) &\equiv \psi_{0,0}(x)\\
& = D_0(x,0)\\
& = S_0(x,0) - S_{-1}(x,0)\\
& = \varphi(x) - 2^{-1}\varphi(2^{-\frac{1}{d}}x)).
\end{align}
Figure \ref{fig:mother} shows the construction of $\varphi$ and $\psi$ in for $d=1,2$.
\begin{figure}[ht!]
  \centering
     \includegraphics[width=2in,height=2in]{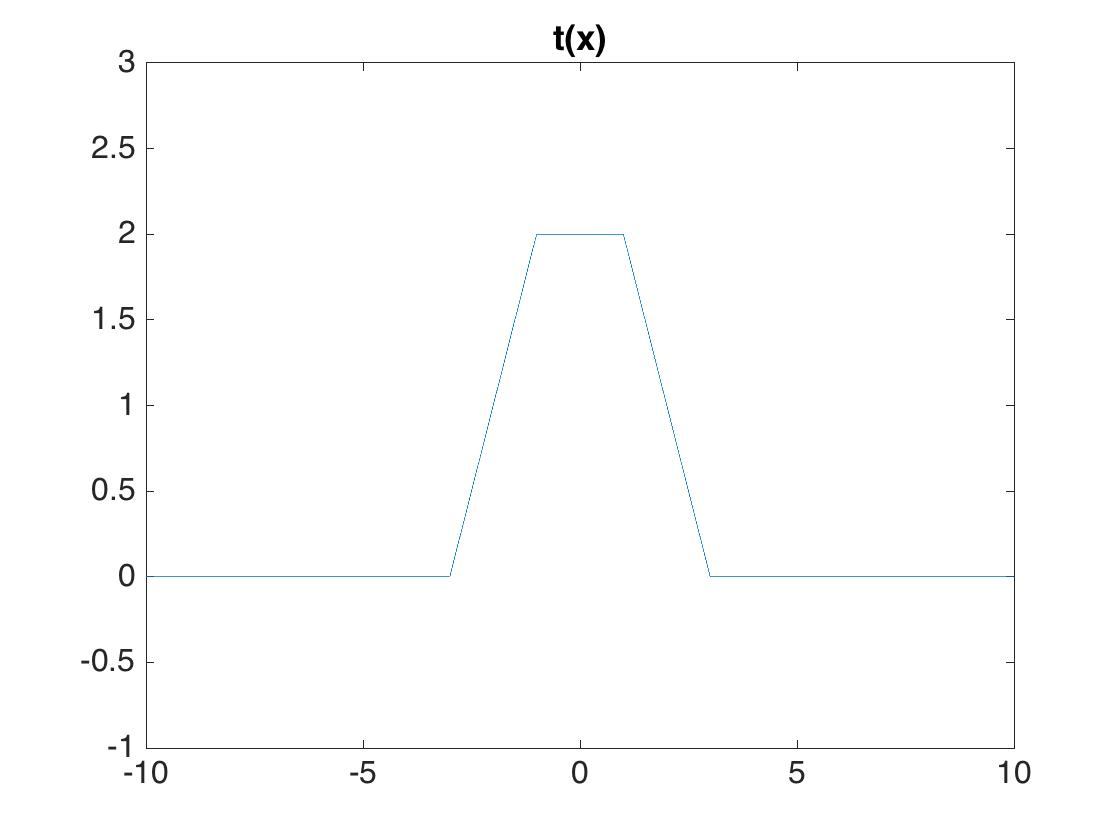}
     \includegraphics[width=2in,height=2in]{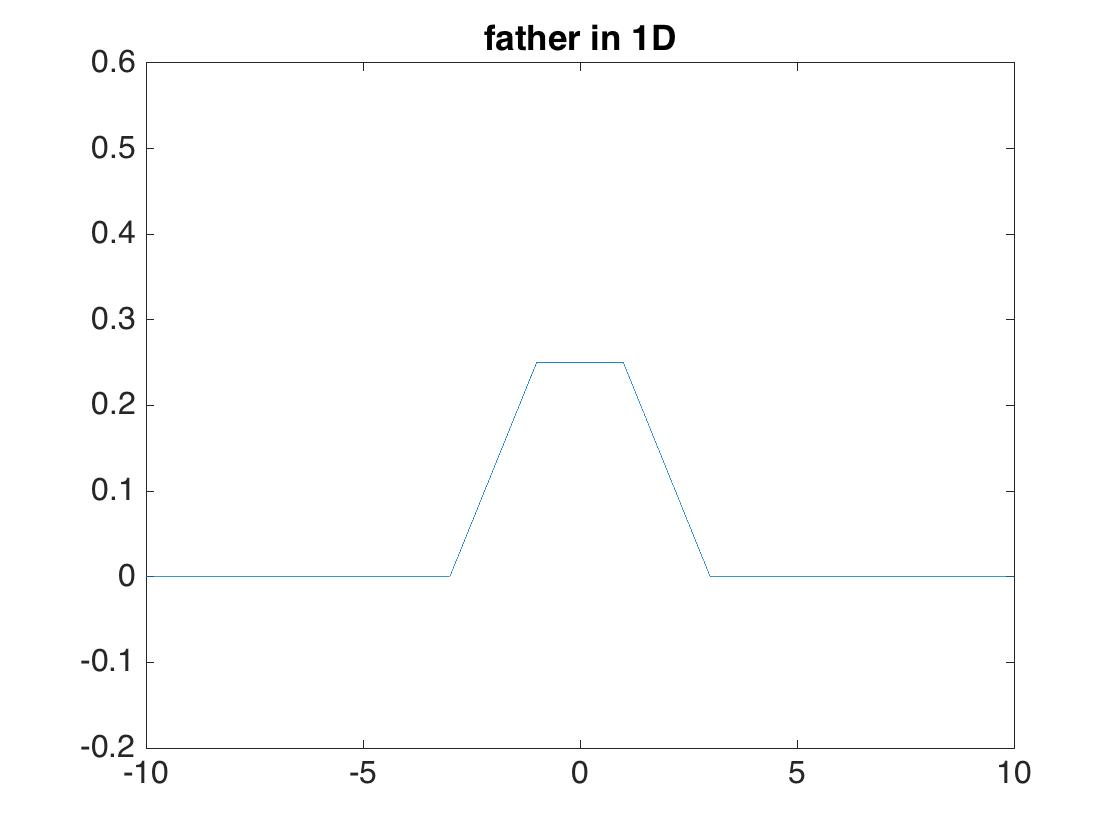}
     \includegraphics[width=2in,height=2in]{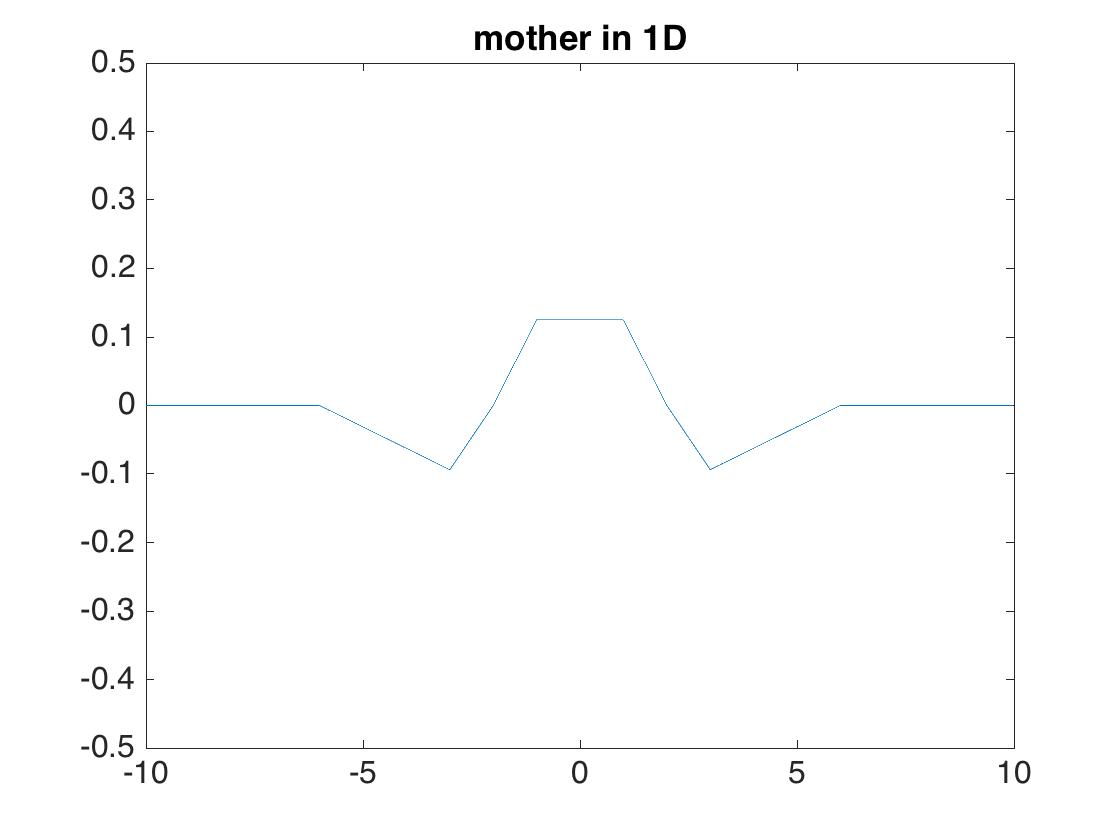}
     \includegraphics[width=2in,height=2in]{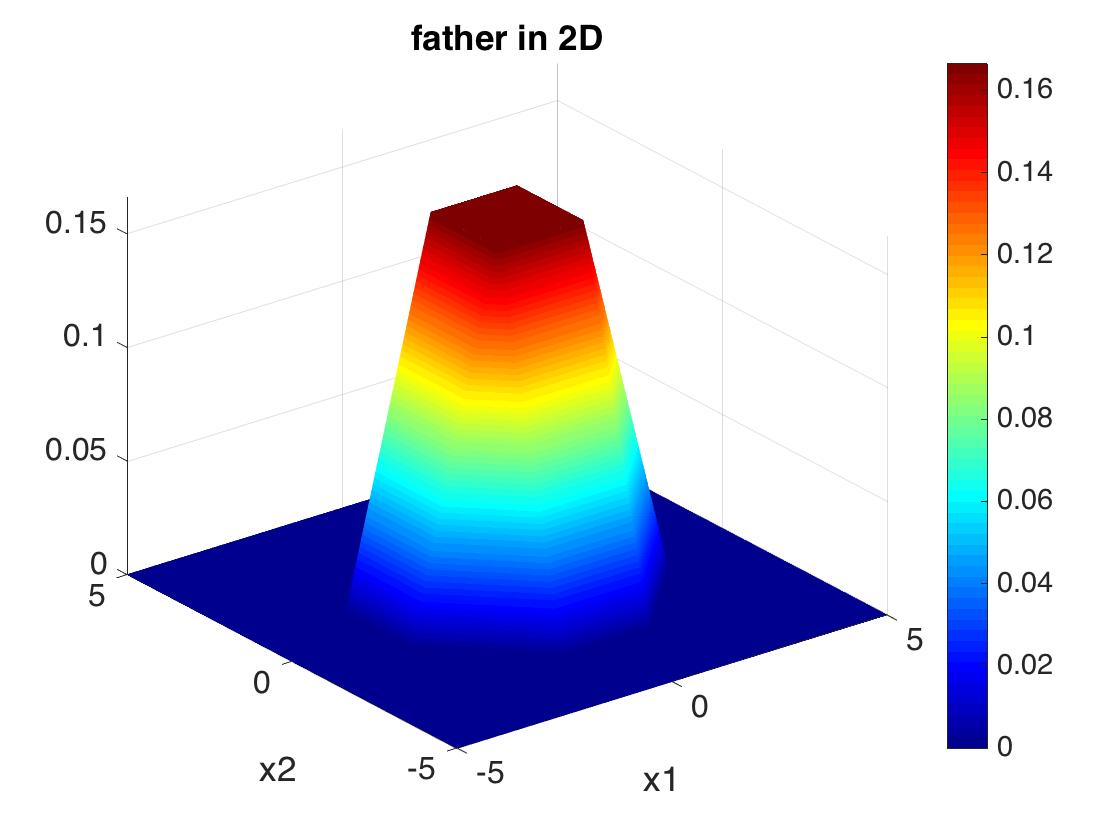}
     \includegraphics[width=2in,height=2in]{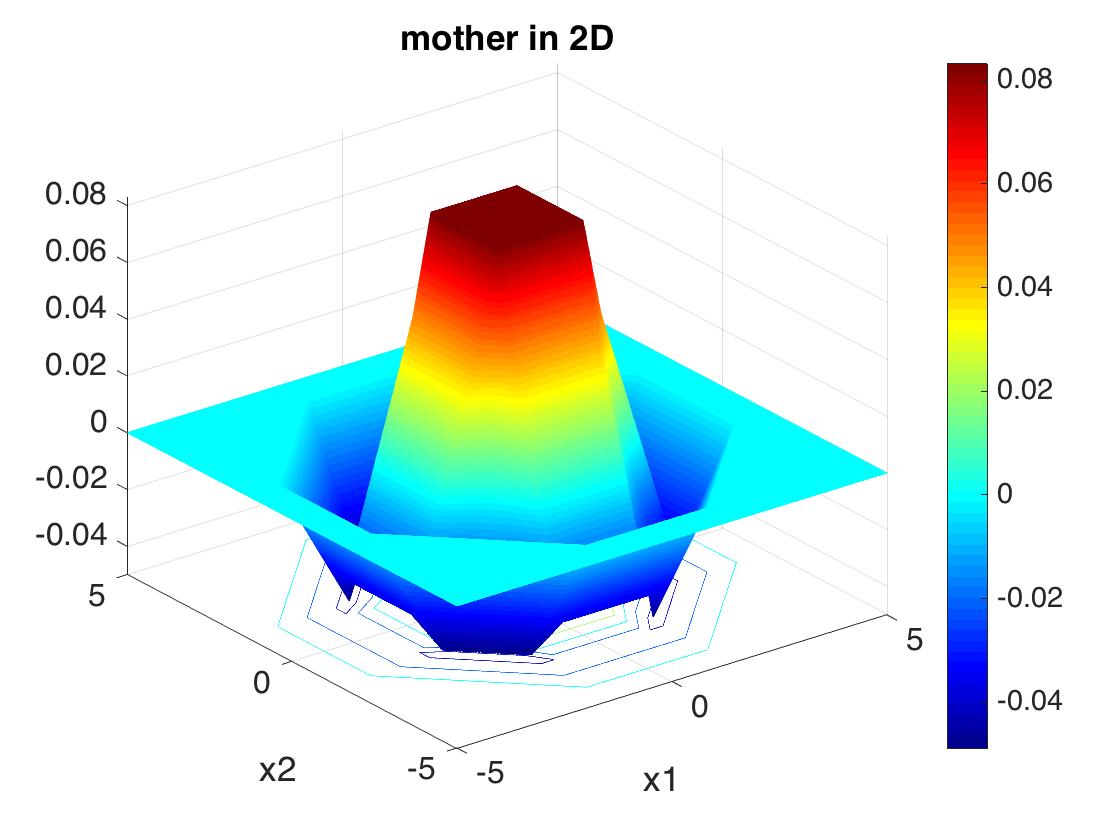}
     \includegraphics[width=2in,height=2in]{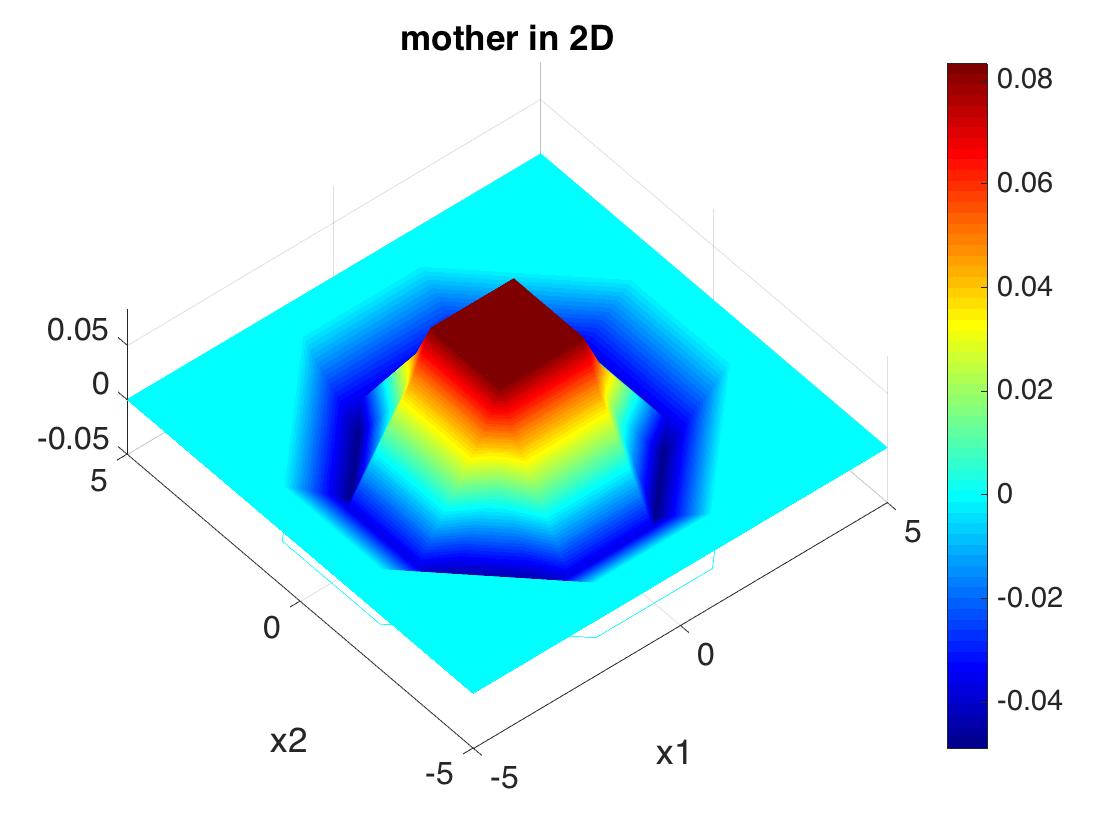}
     \includegraphics[width=2in,height=2in]{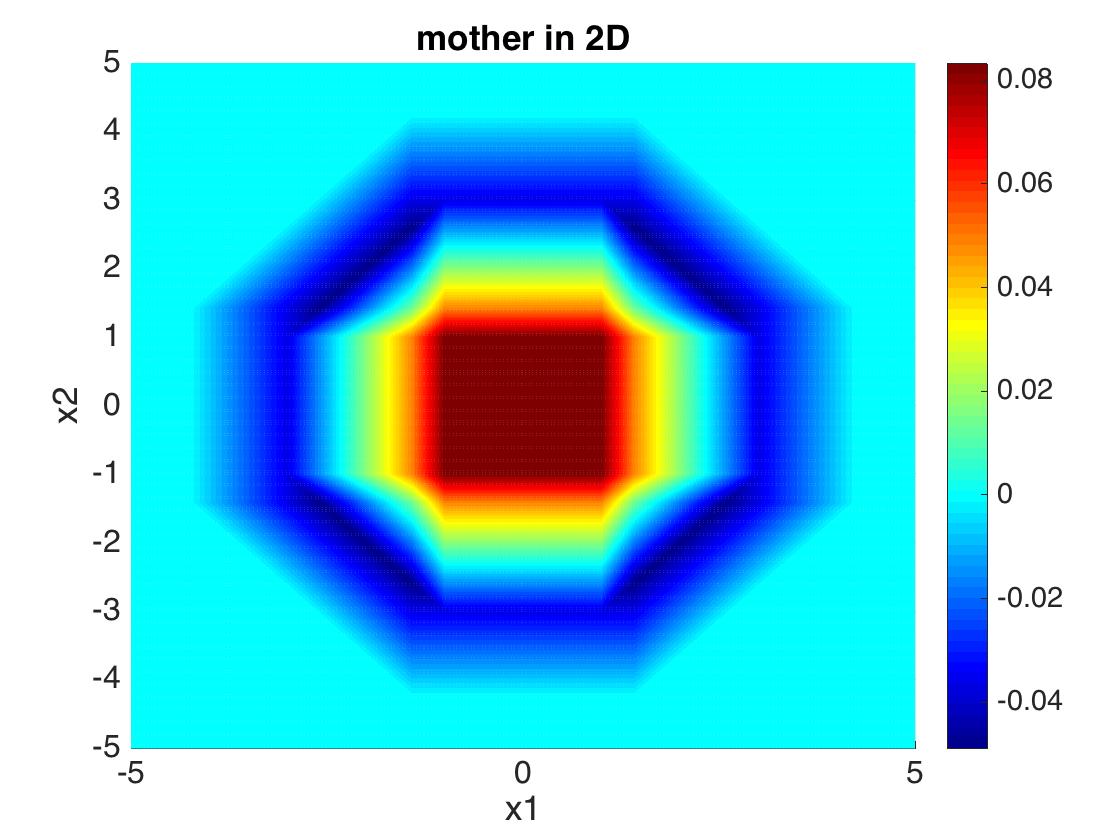}
   \caption{Top row, from left: the trapezoid function $t$, and the functions $\varphi,\psi$ on $\mathbb{R}$. Bottom rows: the functions $\varphi,\psi$ on $\mathbb{R}^2$ from several points of view.}
  \label{fig:mother}
\end{figure}

\begin{remark}
\label{remark:motherFather}
We can see that
\begin{align}
\psi_{k,b}(x) &= 2^{-\frac{k}{2}}D_k(x,b)\\
& = 2^{-\frac{k}{2}}(S_k(x,b) - S_{k-1}(x,b)) \\
& = 2^{-\frac{k}{2}}(2^k\varphi(2^\frac{k}{d}(x-b)) - 2^{k-1}\varphi(2^\frac{k-1}{d}(x-b)))\\
& = 2^\frac{k}{2}\left(\varphi(2^\frac{k}{d}(x-b)) -  2^{-1}\varphi(2^\frac{k-1}{d}(x-b))\right)\\
& = 2^\frac{k}{2}\psi\left(2^\frac{k}{d}(x-b)\right). \label{eq:motherFromFather}
\end{align}
\end{remark}

\begin{remark}
\label{remark:NumUnitsForMother}
With the above construction, $\varphi$ can be computed using a network with $4d$ rectifier units in the first layer and a single unit in the second layer. Hence every wavelet term $\psi_{k,b}$ can be computed using $8d$ rectifier units in the first layer, 2 rectifier units in the second layer and a single linear unit in the third layer. From this, the sum of $k$ wavelet terms can be computed using a network with $8dk$ rectifiers in the first layer, $2k$ rectifiers in the second layer and a single linear unit in the third layer.
\end{remark}

From Theorem \ref{thm:3.25} and the above construction we then get the following lemma:
 \begin{lemma}
 \label{lemma:frame}
 $\{\psi_{k,b} : k \in \mathbb{Z}, b \in 2^{-k}\mathbb{Z}\}$ is a frame of $L_2(\mathbb{R}^d)$.
\end{lemma}

Next, the following lemma uses properties of the 
above frame to obtain point-wise error bounds in approximation of compactly supported functions $f \in C^2$.

\begin{lemma}
\label{lemma:c2Approx}
Let $f \in \L_2(\mathbb{R}^d)$ be compactly supported, twice differentiable and let $\|\nabla^2_f \|_{op}$ be bounded. Then for every $k \in \mathbb{N} \cup \{0 \}$ there exists a combination $f_K$ of terms up to scale $K$ so that for every $x \in \mathbb{R}^d$
 \begin{equation}
 |f(x) - f_K(x)| = O\left(2^{-\frac{2K}{d}}\right).
 \end{equation}
\end{lemma}
The proof is given in Appendix \ref{app:c2Approx}.


\subsection{Creating an atlas}
\label {sec:creatingAtlas}
In this section we specify the number of charts that we would like to have to obtain an atlas for a compact $d$ -dimensional manifold $\Gamma \in \mathbb{R}^m$.

For our purpose here we are interested in a small atlas. We would like the size $C_\Gamma$ of such atlas to depend on the curvature of $\Gamma$: the lower the curvature is, the smaller is the number of charts we will need for $\Gamma$. 

Following the notation of Section \ref{sec:compactManifolds}, let $\delta>0$ so that for all $x\in\Gamma$, $B(x,\delta)\cap \Gamma$ is diffeomorphic to a disc, with a map that is close to the identity. We then cover $\Gamma$ with balls of radius $\frac{\delta}{2}$. The number of such balls that are required to cover $\Gamma$ is 
\begin{equation}
C_\Gamma \le \ceil*{\frac{2^d SA(\Gamma)}{\delta^d}T_d},
\end{equation}
where $SA(\Gamma)$ is the surface area of $\Gamma$, and $T_d$ is the thickness of the covering (which corresponds to by how much the balls need to overlap). 
\begin{remark}
The thickness $T_d$ scales with $d$ however rather slowly: by \cite {conway1993sphere}, there exist covering with $T_d \le d\log d + 5d$. For example, in $d=24$ there exist covering with thickness of $7.9$.
\end{remark}
A covering of $\Gamma$ by such a collection of balls defines an open cover of $\Gamma$ by 
 \begin{equation}
 U_i \equiv B\left(x_i, \delta\right) \cap \Gamma.
 \end{equation}
Let $H_i$ denote the tangent hyperplane tangent to $\Gamma$ at $x_i$. We can now define an atlas by $\{(U_i, \phi_i)\}_{i=1}^{C_\Gamma}$, where $\phi_i$ is the orthogonal projection from $U_i$ onto $H_i$.  

The above construction is sketched in Figure \ref{fig:manifold}.
\begin{figure}[ht!]
  \centering
    \includegraphics[width=3in,height=2in]{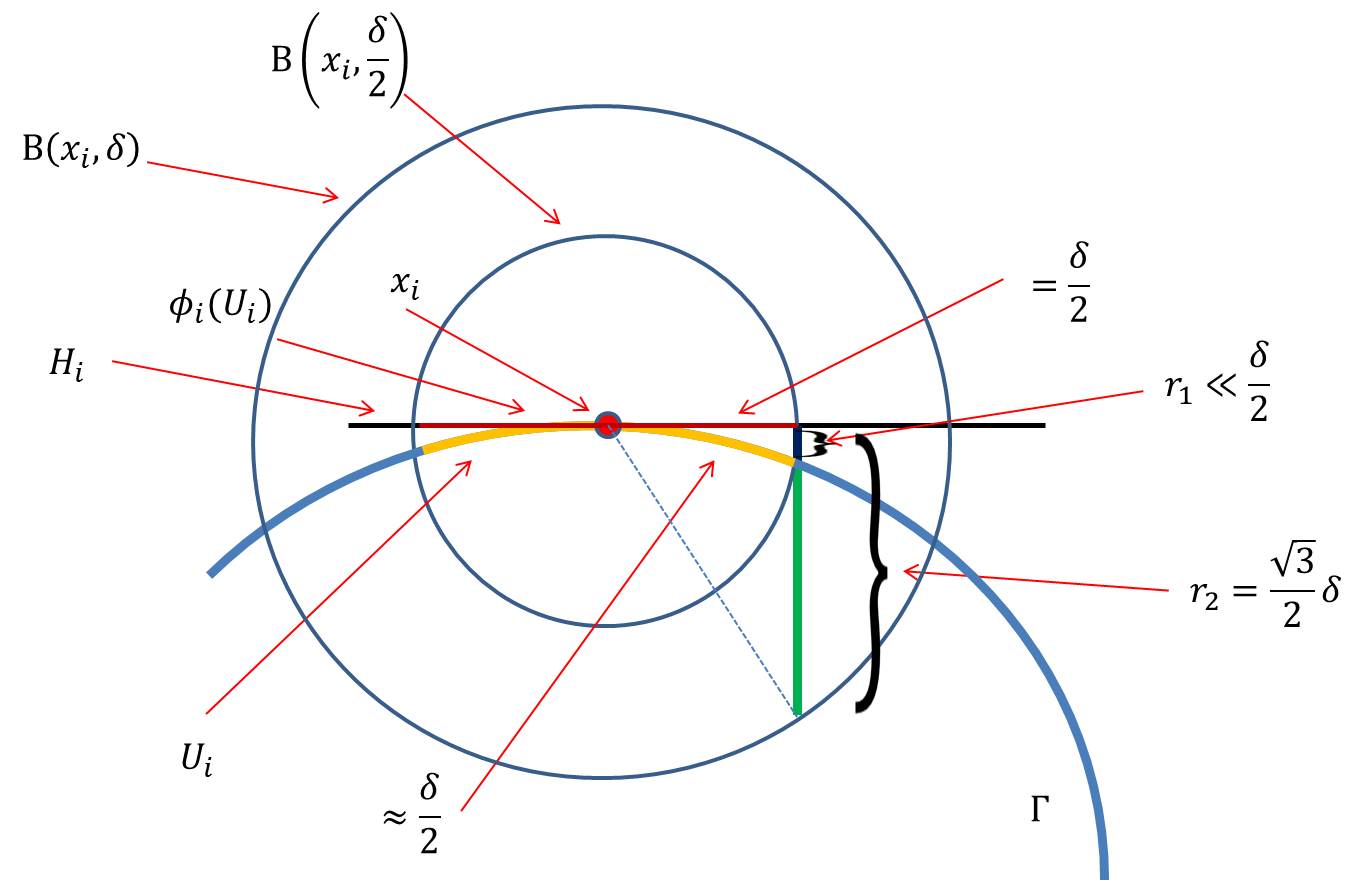}
   \caption{Construction of atlas.}
  \label{fig:manifold}
\end{figure}
Let $\tilde{\phi}_i$ be the extension of $\phi_i$ to $\mathbb{R}^m$, i.e., the orthogonal projection onto $H_i$. The above construction has two important properties, summarized in Lemma \ref{lemma:radii}
\begin{lemma}
\label{lemma:radii}
For every $x \in U_i$,
 \begin{equation}
 \|x - \phi_i(x) \|_2 \le r_1 \le \frac{\delta}{2}
 \end{equation}
and for every $x \in \Gamma\setminus  U_i$ such that $\tilde{\phi}_i(x) \in \phi_i(U_i)$
\begin{equation}
 \|x - \tilde{\phi}_i(x) \|_2 \ge r_2 = \frac{\sqrt{3}}{2}\delta.
 \end{equation}
\end{lemma}


\subsection {Representing a function on manifold as a sum of functions in $\mathbb{R}^d$}
\label{sec:sumOfFunctions}
Let $\Gamma$ be a compact $d$-dimensional manifold in $\mathbb{R}^m$, let $f:\Gamma \rightarrow \mathbb{R}$, let $A =\{(U_i, \phi_i) \}_{i=1}^{C_\Gamma}$ be an atlas obtained by the covering in Section~\ref{sec:creatingAtlas}, and let $\tilde{\phi}_i$ be the extension of $\phi_i$ to $\mathbb{R}^m$. 

$\{U_i\}$ is an open cover of $\Gamma$, hence by Theorem \ref{thm:PartitionOfUnity} there exists a corresponding partition of unity, i.e., a family of compactly supported $C^\infty$ functions $\{\eta_i\}_{i=1}^{C_\Gamma}$ such that
\begin{itemize}
\item $\eta_i:\Gamma \rightarrow [0,1]$
\item $\supp(\eta_i) \subseteq(U_i)$
\item $\sum_i \eta_i=1$
\end{itemize}
 
Let $f_i$ be defined by
\begin{equation} \label{eq:fi}
f_i(x) \equiv f(x)\eta_i(x),
\end{equation}
and observe that $\sum_i f_i=f$. We denote the image  $\phi_i(U_i)$  by $I_i$. Note that  $ I_i \subset H_i$, i.e., $I_i$ lies in a $d$-dimensional hyperplane $H_i$ which is isomorphic to $\mathbb{R}^d$. We define $\hat{f}_i$ on $\mathbb{R}^d$ as
\begin{equation}
\label{eq:fHat}
\hat{f}_i(x) = \bigg\{\begin{tabular}{cc}
$f_i(\phi^{-1}(x))$ & $x \in I_i$\\
$0$ & otherwise
\end{tabular}
\end{equation}
and observe that $\hat{f}_i$ is compactly supported. This construction  gives the following Lemma
\begin{lemma}
\label{lemma:sumFi}
For all $x \in \Gamma$, 
\begin{equation}
\sum_{\{i:x \in U_i\}} \hat{f}_i(\phi_i(x)) = f(x).
\end{equation}
\end{lemma}
Assuming $\hat{f}_i \in L_2(\mathbb{R}^d)$, by Lemma \ref{lemma:frame} it has a wavelet expansion using the frame that was constructed in Section \ref{sec:wavConstruction}.


\subsection {Extending the wavelet terms in the approximation of $\hat{f}_i$ to $\mathbb{R}^m$}\label{sec:extension}
Assume that  $\hat{f}_i \in L_2(\mathbb{R}^d)$ and let
\begin{equation} \label{eq:fiWaveletApprox}
\hat{f_i} = \sum_{(k,b)}\alpha_{k,b}\psi_{k,b},
\end{equation}
be its wavelet expansion, where $\alpha_{k,b} \in \mathbb{R}$ and $\psi_{k,b}$ is defined on $\mathbb{R}^d$.

We now show how to extend each $\psi_{k,b}$ to $\mathbb{R}^m$. Let's assume (for now) that the coordinate system is such that the first $d$ coordinates are the local coordinates (i.e., the coordinates on $H_i$) and the remaining $m-d$ coordinates are of the directions which are orthogonal to $H_i$.

Intuitively, we would like to extend the wavelet terms on $H_i$ to $\mathbb{R}^m$ so that they remain constant until they "hit" the manifold, and then die off before they "hit" the manifold again.  By Lemma \ref{lemma:radii} it therefore suffices to extend each  $\psi_{k,b}$ to $\mathbb{R}^m$ so that in each of the $m-d$ orthogonal directions, $\psi_{k,b}$  will be constant in $[-\frac{r_1}{\sqrt{m-d}}, \frac{r_1}{\sqrt{m-d}}]$ and will have a support which is contained in $[-\frac{r_2}{\sqrt{m-d}}, \frac{r_2}{\sqrt{m-d}}]$.

Recall  from Remark \ref{remark:motherFather} that each of the wavelet terms $\psi_{k,b}$ in Equation~\eqref{eq:fiWaveletApprox} is defined on $\mathbb{R}^d$ by
\begin{align}
\psi_{k,b}(x) &= 2^\frac{k}{2}\left(\varphi(2^\frac{k}{d}(x-b))- 2^{-1}\varphi(2^\frac{k-1}{d}(x-b))\right)\\
\end{align}
and recall that as in Equation~\eqref{eq:father}, the scaling function $\varphi$ was defined on on $\mathbb{R}^d$ by 
\begin{equation}
\varphi(x) = C_d\rect\left(\sum_{j=1}^d t(x_j) -2(d-1)\right).
\end{equation}
We extend $\psi_{k,b}$ to $\mathbb{R}^m$ by
\begin{equation}\label{eq:extendedMother}
\psi_{k,b}(x) \equiv 2^\frac{k}{2}\left(\varphi_r(2^\frac{k}{d}(x-b))- 2^{-1}\varphi_r(2^\frac{k-1}{d}(x-b))\right),
\end{equation}
where
\begin{equation}\label{eq:extendedFather}
\varphi_r(2^\frac{k}{d}(x-b)) \equiv C_d\rect\left(\sum_{j=1}^d t(2^\frac{k}{d}(x_j-b_j))  +\sum_{j=d+1}^m t_r(x_j) -2(m-1)\right),
\end{equation}
and $t_r$ is a trapezoid function which is supported on $[-\frac{r_2}{\sqrt{m-d}}, \frac{r_2}{\sqrt{m-d}}]$ and its top (small) base is between $[-\frac{r_1}{\sqrt{m-d}}, \frac{r_1}{\sqrt{m-d}}]$ and has height 2. 
This definition of $\psi_{k,b}$ gives it a constant height for distance $r_1$ from $H_i$, and then a linear decay, until it vanishes at distance $r_2$.
%
Then by construction we obtain the following lemma
\begin{lemma}
\label{lemma:outsideSupport}
For every chart $(U_i, \phi_i)$ and every $x \in \Gamma \setminus U_i$ such that $\tilde{\phi}_i(x) \in \phi_i(U_i)$, $x$ is outside the support of every wavelet term corresponding to the $i$'th chart.
\end{lemma}

\begin{remark}
\label{remark:extensionUnits}
Since the $m-d$ additional trapezoids in Equation~\eqref{eq:extendedFather} do not scale with $k$ and shift with $b$, they can be shared across all scaling terms in Equations~\eqref{eq:extendedMother} and~\eqref{eq:fiWaveletApprox}, so that the extension of the wavelet terms from $\mathbb{R}^d$ to $\mathbb{R}^m$ can be computed with $4(m-d)$ rectifiers.
\end{remark}

Finally, in order for this construction to work for all $i=1,...,C_\Gamma$ the input $x\in \mathbb{R}^m$ of the network can be first mapped to $\mathbb{R}^{mC_\Gamma }$ by a linear transformation so that the each of the $C_\Gamma$ blocks of $m$ coordinates gives the local coordinates on $\Gamma$ in the first $d$ coordinates and on the orthogonal subspace in the remaining $m-d$ coordinates.  These maps are essentially the orthogonal projections $\tilde{\phi}_i$.




\section {Specifying the required size of the network} 
\label{sec:counting}
In the construction of Section \ref{sec:Main}, we approximate a function $f \in L_2(\Gamma)$ using a depth $4$ network, where the first layer computes the local coordinates in every chart in the atlas, the second layer computes $\rect$ functions that are to form trapezoids, the third layer computes scaling functions of the form $\varphi(2^\frac{k}{d}(x-b))$ for various $k,b$ and the fourth layer consists of a single node which computes
 \begin{equation}
 \hat{f} = \sum_{i=1}^{C_\Gamma} \sum_{(k,b)}\psi^{(i)}_{k,b},
 \end{equation}
 where $\psi^{(i)}_{k,b}$ is a wavelet term on the $i$'th chart. This network is sketched in Figure \ref{fig:netSketch}.
\begin{figure}[ht!]
  \centering
    \includegraphics[width=4.5in,height=2.5in]{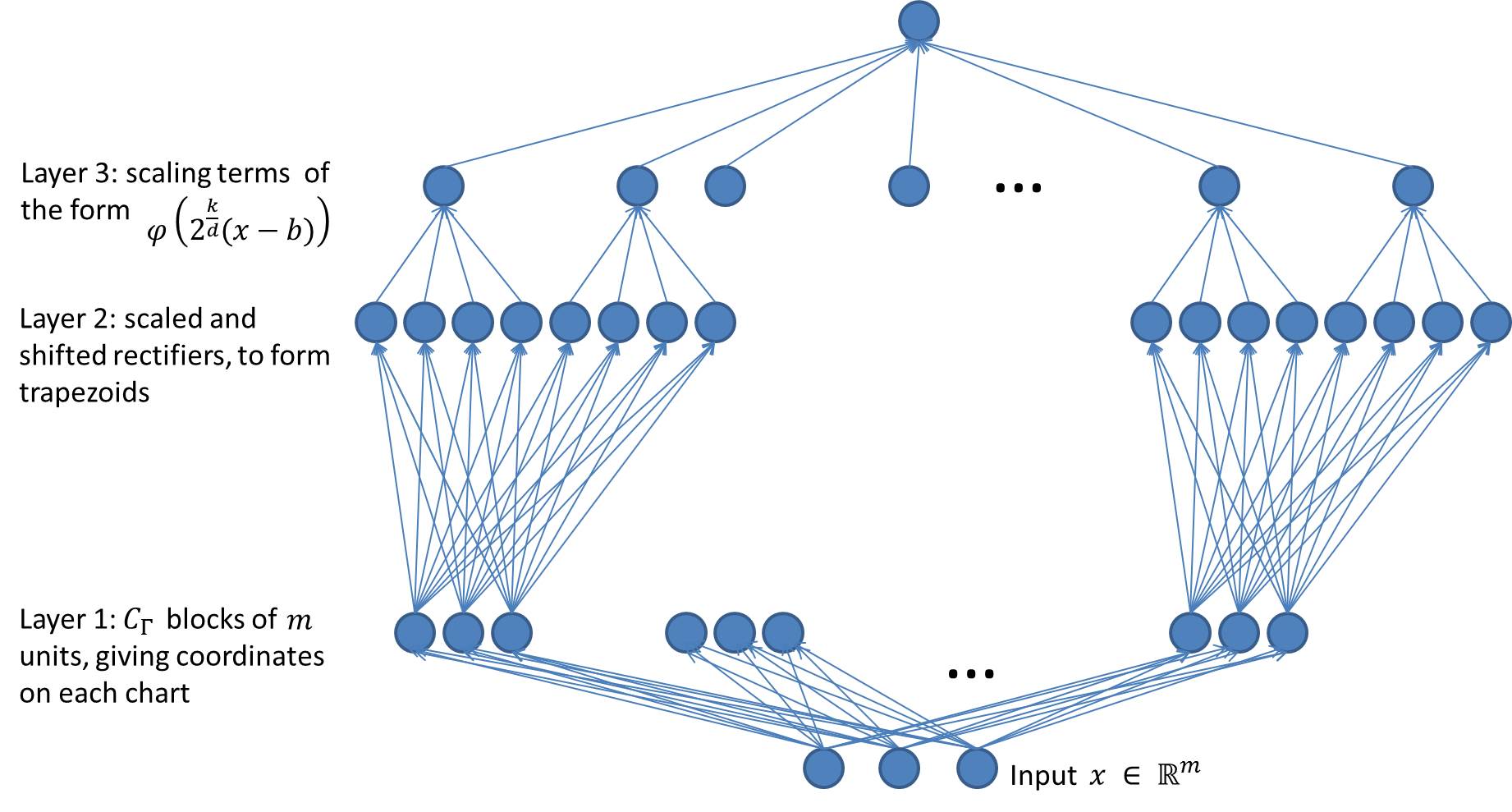}
   \caption{A sketch of the network.}
  \label{fig:netSketch}
\end{figure}

From this construction, we obtain the following theorem, which is the main result of this work:
\begin{theorem}\label{thm:main}
Let $\Gamma$ be a $d$-dimensional manifold in $\mathbb{R}^m$, and let $f \in L_2(\Gamma)$. Let $\{(U_i, \phi_i)\}$ be an atlas of size $C_\Gamma$ for $\Gamma$, as in Section \ref{sec:creatingAtlas}. Then $f$ can be approximated using a 4-layer network with
$mC_\Gamma$ linear units in the first hidden layer,
$8d\sum_{i=1}^{C_\Gamma}N_i +4C_\Gamma(m-d) $
rectifier units in the second hidden layer, $2\sum_{i=1}^{C_\Gamma}N_i$ rectifier units in the third layer and a single linear unit in the fourth (output) layer, where $N_i$ is the number of wavelet terms that are used for approximating $f$ on the $i$'th chart.
\end{theorem}
\begin{proof}
As in Section \ref{sec:sumOfFunctions}, we construct functions $\hat{f}_i$ on $\mathbb{R}^d$ as in Equation (\ref{eq:fHat}), which, by Lemma \ref{lemma:sumFi}, have the property that for every $x \in \Gamma$, $\sum_{\{i:x \in U_i\}} \hat{f}_i(\phi_i(x)) = f(x)$. 
The fact that $\hat{f}_i$ is compactly supported means that its wavelet approximation converges to zero outside $\phi_i(U_i)$. Together with Lemma \ref{lemma:outsideSupport}, we then get that an approximation of $f$ is obtained by summing up the approximations of all the $\hat{f}_i$'s. 

A first layer of the network will consist $mC_\Gamma$ linear units and will compute the map as in the last paragraph of Section \ref{sec:extension}, i.e., linearly transform the input to $C_\Gamma$ blocks, each of dimension $m$, so that in each block $i$ the first $d$ coordinates are with respect to the tangent hyperplane $H_i$ (i.e., will give the representation $\tilde{\phi}_i(x)$) and the remaining $m-d$ coordinates are with respect to directions orthogonal to $H_i$.

For each $i=1,..,C_\Gamma$, we approximate each $\hat{f}_i$ to some desired approximation level $\delta$ using $N_i < \infty$ wavelet terms. 
By Remark \ref{remark:NumUnitsForMother}, $\hat{f}_i$ can be approximated using $8dN_i$ rectifiers in the second layer, $2N_i$ rectifiers in the third layer and a single unit in the fourth layer.  
By Remark \ref{remark:extensionUnits}, on every chart the wavelet terms in all scales and shifts can be extended to $\mathbb{R}^m$ using (the same) $4(m-d)$ rectifiers in the second layer. 

Putting this together we get that to approximate $f$ one needs a 4-layer network with
$mC_\Gamma$ linear units in the first hidden layer
$8d\sum_{i=1}^{C_\Gamma}N_i  + 4C_\Gamma(m-d)$
rectifier units in the second hidden layer, $2\sum_{i=1}^{C_\Gamma}N_i$ rectifier units in the third layer and a single linear unit in the fourth (output) layer.
\end{proof}

\begin{remark}
For sufficiently small radius $\delta$ in the sense of section~\ref{sec:compactManifolds}, the desired properties of $\hat{f}_i$ (i.e., being in $L_2$ and possibly having sparse coefficient or being twice differentiable) imply similar properties of $f$.
\end{remark}

\begin{remark}
We observe that the dependence on the dimension $m$ of the ambient space in the first and second layers is through $C_\Gamma$, which depends on the curvature of the manifold. The number $N_i$ of wavelet terms in the $i$'th chart affects the number of units in the second layer only through the dimension $d$ of the manifold, not through $m$. The sizes of the third and fourth layers do not depend on $m$ at all.
\end{remark}

Finally, assuming regularity conditions on the $\hat{f}_i$, allows us to bound the number $N_i$ of wavelet terms needed for the approximation of  $\hat{f}_i$. In particular, we consider two specific cases: $\hat{f}_i \in \mathcal{L}_1$ and $\hat{f}_i \in C^2$, with bounded second derivative. 

\begin{corollary}\label{cor:sparse}
If $\hat{f}_i \in \mathcal{L}_1$ (i.e., $\hat{f}_i$ has expansion coefficients in $l_1$), then by Theorem \ref{thm:BarronGreedy}, $\hat{f}_i$ can be approximated by a combination $\hat{f}_{i,N_i}$ of $N_i$ wavelet terms so that
\begin{equation}
\|\hat{f}_i- \hat{f}_{i,N_i}\|_2 \le  \frac{\| \hat{f}_i\|_{\mathcal{L}_1}}{\sqrt{N_i+1}}.
\end{equation}

Consequently, denoting the output of the net by $\tilde{f}$,  $N \equiv\max_i\{N_i\}$ and $M \equiv \max_i \|\hat{f}_i \|_{\mathcal{L}_1}$, we obtain 
\begin{equation}
\|f-\tilde{f} \|_2^2 \le \frac{C_\Gamma M}{N+1},
\end{equation}
using $c_1 + c_2N$ units, where 
$c_1=C_\Gamma(m +  4(m-d)) + 1$ and
$c_2 = (8d+2)C_\Gamma $. 
\end{corollary}

\begin{corollary}\label{cor:c2}
If for each $i$ $\hat{f_i}$'s is twice differentiable and $\|\nabla^2_{f_i}\|_{op}$ is bounded, then by Lemma \ref{lemma:c2Approx} $\hat{f_i}$ can be approximated by $\hat{f}_{K,i}$ using all terms up to scale $K$ so that for every $x \in \mathbb{R}^d$
\begin{equation}
|\hat{f}_i(x)- \hat{f}_{i,K}(x)| = O \left(2^{-\frac{2K}{d}}\right).
\end{equation}

Observe that the grid spacing in the $k$'th level is $2^{-\frac{k}{d}}$. Therefore, since $f$ is compactly supported, there are $O\left(\left( 2^\frac{k}{d} \right)^d\right)=O\left(2^k\right)$ terms in the $k$'th level. Altogether, on the $i$'th chart there are $O\left(2^{K+1}\right)$ terms in levels less than $K$. Writing $N\equiv 2^{K+1}$, we get a point-wise error rate of $N^{-\frac{2}{d}}$ using $c_1 + c_2N$ units, where 
$c_1=C_\Gamma(m +  4(m-d)) + 1$ and
$c_2 = (8d+2)C_\Gamma $. 
\end{corollary}

\begin{remark}
The unit count in Theorem~\ref{thm:main} and Corollaries \ref{cor:sparse} and \ref{cor:c2} is overly pessimistic, in the sense that we assume that the sets of wavelet terms in the expansion of $\hat{f}_i$, $\hat{f}_j$ do not intersect, where $i,j$ are chart indices. A tighter bound can be obtained if we allow wavelet functions be shared across different charts, in which case the term $C_\Gamma\sum N_i$ in Theorem~\ref{thm:main} can be replaced by the total number of distinct wavelet terms that are used on all charts, hence decreasing the constant $c_2$. In particular, in Corollary  \ref{cor:c2} we are using all terms up to the $K$'th scale on each chart. In this case the constant $c_2=8d+2$.
\end{remark}

\begin{remark}
The linear units in the first layer can be simulated using ReLU units with large positive biases, and adjusting the biases of the units in the second layer. Hence the first layer can contain ReLU units instead of linear units.
\end{remark}
\section{Conclusions}
\label{sec:conclusions}
The construction presented in this manuscript can be divided to two main parts: analytical and topological. In the analytical part, we constructed a wavelet frame if $L_2(\mathbb{R}^d)$, where the wavelets are computed from Rectified Linear units. In the topological part, given training data on a $d$-dimensional manifold $\Gamma$ we constructed an atlas and represented any function on $\Gamma$ as sum of functions that are defined on the charts. We then used Rectifier units to extend the wavelet approximation of the functions from $\mathbb{R}^d$ to the ambient space $\mathbb{R}^m$. This construction allows us to state the size of a depth 4 neural net given a function $f$ to be approximated on the manifold $\Gamma$. We show how the specified size depends on the complexity of the function (manifested in the number of wavelet terms in its approximation) and the curvature of the manifold (manifested in the size of the atlas). In particular, we take advantage of the fact that $d$ can possibly be much smaller than $m$ to construct a network with size that depends more strongly on $d$.
In addition, we also obtain squared error rate in approximation of functions with sparse wavelet expansion and point-wise error rate for twice differentiable functions.

The network architecture and corresponding weights presented in this manuscript is hand-made, and is such that achieves the approximation properties stated above. However, it is reasonable to assume that such network is unlikely to be the result of a standard training process. Hence, we see the importance of the results presented in this manuscript by describing the theoretical approximation capability of neural nets, and  not by describing trained nets which are used in practice.   

Several extensions of this work can be considered. First, a more efficient wavelet representation can be obtained on each chart if one allows its wavelets to be non-isotropic (that is, to scale differently in every dimension) and not necessarily axis aligned, but rather, to correspond to the level sets of the function being approximated. When the function is relatively constant in certain directions, the wavelet terms can be "stretched" in these directions. Such thing can be done using curvelets.

Second, we conjecture that in the representation obtained as an output of convolutional and pooling layers, the data concentrates near a collection of low dimensional manifolds embedded in a high dimensional space, which is our starting point in the current manuscript. 
We think that this is a result of the application of the same filters to all data points. Assuming our conjecture is true, one can apply our construction to the output of convolutional layers, and by that obtain a network topology which is similar to standard convolutional networks, namely fully connected layers on top of convolutional ones. This will make or arguments here applicable to cases where the data in its initial representation does not concentrate near low dimensional manifold, but its hidden representation does. 

Finally, we remark that the choice of using rectifier units to construct our wavelet frame is convenient, however somewhat arbitrary. Similar wavelet frames can be constructed by any function (or combination of functions) that can be used to construct ``bump'' functions i.e., functions which are localized and have fast decay. For example, general sigmoid functions $\sigma:\mathbb{R} \rightarrow \mathbb{R}$, which are monotonic and have the properties
\begin{equation}
\lim_{x \rightarrow -\infty}\sigma(x)=0 \text { and } \lim_{x \rightarrow \infty}\sigma(x)=1
\end{equation}
can used to construct a frame in a similar way, by computing ``smooth'' trapezoids. Recall also that by Remark \ref{remark:equivalentFrames}, any two such frames are equivalent.

\section*{Acknowledgements}
The authors thank Stefan Steinerberger, Roy Lederman for their help, and to Andrew Barron, Ed Bosch, Mark Tygert and Yann LeCun for their comments. Alexander Cloninger is supported by NSF Award No. DMS-1402254.

\bibliography{draft}{}
\bibliographystyle{plain}


\appendix
\newpage

\section{Equivalence of representations in different wavelet frames}
\label{app:equivFrame}
Consider to frames $\{\psi_{k,b} \}$ and $\{\psi'_{k,b} \}$. Any element $\psi'_{k',b'}$ can be represented as
\begin{equation}
\psi'_{k',b'} = \sum_{k,b}\langle \psi'_{k',b'},\widetilde{\psi}_{k,b} \rangle \psi_{k,b}.
\end{equation}
Observe that in case $k\approx k'$, the inner product is of large magnitude only for a small number of $b'$s. In case $k \ll k'$ or $k \gg k'$, the inner product is between peaked function which integrates to zero and a flat function, hence has small magnitude. This idea is formalized  in a more general form in Section $4.7$ in \cite{deng2009harmonic}.

\section{Proof of Lemma \ref{lemma:conditions}}.
\label{app:conditions}
\begin{proof}
In order to show that the family $\{S_k\}$ in Equation \eqref{eq:Sdef} is a valid family of averaging kernel functions, we need to verify that conditions $3.14-3.19$ in \cite{deng2009harmonic} are satisfied. Here $\rho(x,b)$ is the volume of the smallest Euclidean ball which contains $x$ and $b$, namely $\rho(x,b) = c\|x-b\|^d$, for some constant $c$. Our goal is to show that there exist constants $C \le \infty$, $\sigma >0$ and $\epsilon >0$ such that for every $k\in \mathbb{Z}$, and $x,x',b,b' \in \mathbb{R}^d$

\begin{itemize}
\item $3.14$: 
\begin{equation} \label{eq:cond3.14}
S_k(x,b) \le C\frac{2^{-k\epsilon}}{(2^{-k}+\rho(x,b))^{1+\epsilon}},
\end{equation}
\begin{proof}
WLOG we can assume $b=0$, and let $\epsilon$ be arbitrary positive number.
It can be easily verified that there exists a constant $C'$ such that 
\begin{equation}
\varphi(x) \le \frac{C'}{\left(c^{-1} + \|x\|^d \right)^{1+\epsilon}}.
\end{equation}
Then 
\begin{align}
S_k(x,0) &=2^{k}\varphi\left(2^\frac{k}{d}x\right)\\
&\le C'\frac{2^k}{\left(c^{-1} + 2^{k}\|x\|^d \right)^{1+\epsilon}}\\
&= C'\frac{2^{k(1+\epsilon)}2^{-k\epsilon}}{\left(c^{-1} + 2^{k}\|x\|^d \right)^{1+\epsilon}}\\
&= C'\frac{2^{-k\epsilon}}{\left(c^{-1}2^{-k} + \|x\|^d \right)^{1+\epsilon}}\\
&= C_1\frac{ 2^{-k\epsilon}}{\left(2^{-k} + \rho(x,0) \right)^{1+\epsilon}},
\end{align}
where $C_1 = c^{1+\epsilon}C'$.
\end{proof}

\item $3.15,3.16$:  Since $S_k(x,b)$ depends only on $x-b$ and is symmetric about the origin, it suffices to prove only $3.15$. We want to show that if $\rho(x,x') \le \frac{1}{2A}(2^{-k}+\rho(x,b))$ then 
\begin{equation}
|S_k(x,b)-S_k(x',b)| \le C\left(\frac{\rho(x,x')}{2^{-k}+\rho(x,b)} \right)^\sigma\frac{2^{-k\epsilon}}{(2^{-k}+\rho(x,b))^{1+\epsilon}}.
\end{equation}
\begin{proof}
WLOG $b=0$; we will prove for every $x,x'$. Let $\epsilon$ be arbitrary positive number, and let $\sigma = \frac{1}{d}$.
By the mean value theorem we get
\begin{equation}
\frac{|S_k(x,0)-S_k(x',0)|}{\rho(x,x')^\sigma} \le \max_{z_k \text{ between } x,x'} \frac{1}{c}\|\nabla_x (S_k(z_k,0))\|.
\end{equation}
Denote
\begin{equation}
F(x) \equiv \|\nabla_x (S_0(x,0))\|.
\end{equation}
Then
\begin{equation}
\|\nabla_x (S_k(x,0))\| = 2^k2^\frac{k}{d}F\left(2^\frac{k}{d}x\right).
\end{equation}

As in the proof of condition $3.14$, it can be easily verified that there exists a constant $C'$ such that
\begin{equation}
F(x) \le C'\frac{1}{\left(c^{-1}+\|x\|^d\right)^\sigma}\frac{1}{(c^{-1}+\|x\|^d)^{1+\epsilon}}.
\end{equation}
We then get
\begin{align}
\frac{|S_k(x,b)-S_0(x',b)|}{\rho(x,x')^\sigma} &= \frac{1}{c}\|\nabla_x (S_k(z_k,0))\|\\
&= 2^k2^\frac{k}{d}F \left(2^\frac{k}{d}\right)\\
&\le C'\frac{2^\frac{k}{d}}{\left(c^{-1}+2^k\|x\|^d\right)^\sigma}\frac{2^k}{(c^{-1}+2^k\|x\|^d)^{1+\epsilon}}\\
&= C'\frac{2^\frac{k}{d}}{\left(c^{-1}+2^k\|x\|^d\right)^\sigma}\frac{2^{k(1+\epsilon)}2^{-k\epsilon}}{(c^{-1}+2^k\|x\|^d)^{1+\epsilon}}\\
&= C'\frac{1}{\left(c^{-1}2^{-k}+\|x\|^d\right)^\sigma}\frac{2^{-k\epsilon}}{(c^{-1}2^{-k}+\|x\|^d)^{1+\epsilon}}\\
&= C_2\frac{1}{\left(2^{-k}+\rho(x,0)\right)^\sigma}\frac{2^{-k\epsilon}}{(2^{-k}+\rho(x,0))^{1+\epsilon}},
\end{align}
where $C_2 = c^{\sigma + 1+\epsilon}C'$.
\end{proof}

\item $3.17,3.18$: 
Since $S_k(x,b)$ depends only on $x-b$ and is symmetric about the origin, it suffices to prove only $3.17$.
\begin{proof}
By Equation \eqref{eq:father}
\begin{equation}
\int_{\mathbb{R}^d}\varphi(x)dx=1
\end{equation}
and consequently for every $k\in \mathbb{Z}$ and $b \in \mathbb{R}^d$
\begin{equation}
\int_{\mathbb{R}^d}S_k(x,b)dx=1.
\end{equation}
\end{proof}

\item $3.19$: we want to show if $\rho(x,x') \le \frac{1}{2A}(2^{-k}+\rho(x,b))$ and $\rho(b,b') \le c(2^{-k}+\rho(x,b))$ then
\begin{align}
&|S_k(x,b)-S_k(x',b) - S_k(x,b') + S_k(x',b')|\\
& \le C\left(\frac{\rho(x,x')}{2^{-k}+\rho(x,b)} \right)^\sigma\left(\frac{\rho(b,b')}{2^{-k}+\rho(x,b)} \right)^\sigma\frac{2^{-k\epsilon}}{(2^{-k}+\rho(x,b))^{1+\epsilon}}.
\end{align}

\begin{proof}
We will prove for all $x,x',b,b'$. Let $\sigma = \frac{1}{d}$.  Observe that 
\begin{align}
&\frac{|S_k(x,b)-S_k(x',b) - S_k(x,b') + S_k(x',b')|}{\rho(x,x')^\sigma\rho(b,b')^\sigma} \\
&\le \frac{|\frac{|S_k(x,b)-S_k(x',b)|}{\rho(x,x')^\sigma} + \frac{|S_k(x,b') + S_k(x',b')|}{\rho(x,x')^\sigma}   |}{\rho(b,b')^\sigma}\\
\end{align}
Denote
\begin{equation}
F(b) \equiv \frac{|S_k(x,b)-S_k(x',b)|}{\rho(x,x')^\sigma}.
\end{equation}
Then by applying the mean value theorem twice we get
\begin{align}
&\frac{|\frac{|S_k(x,b)-S_k(x',b)|}{\rho(x,x')^\sigma} + \frac{|S_k(x,b') + S_k(x',b')|}{\rho(x,x')^\sigma}   |}{\rho(b,b')^\sigma}\\
&= \frac{|F(b)-F(b')|}{\rho(b,b')^\sigma}\\
& \frac{1}{c}\le \max_{z \text{ between } b,b'}\nabla_b(F(z))\\
&= \frac{1}{c}\max_{z \text{ between } b,b'} \nabla_b\left(\frac{|S_k(x,z)-S_k(x',z)|}{\rho(x,x')^\sigma}\right)\\
&\frac{1}{c^2}\le \max_{z \text{ between } b,b'}\max_{z' \text{ between } x,x'} \|\nabla^2_{x,b}(S_k(z',z))\|
\end{align}

From this, we can see that 
Since $S_k$ is compactly supported and bounded, there exist compactly supported function $\xi(x)$ such that 
\begin{align}
&\frac{|S_0(x,b)-S_0(x',b) - S_0(x,b') + S_0(x',b')|}{\rho(x,x')^\sigma\rho(b,b')^\sigma} \\ &\le \xi(x-b) + \xi(x-b'),
\end{align}
and consequently
\begin{align}
&\frac{|S_k(x,b)-S_k(x',b) - S_k(x,b') + S_k(x',b')|}{\rho(x,x')^\sigma\rho(b,b')^\sigma} |\\ &\le 2^k2^\frac{2k}{d}\left(\xi\left(2^\frac{k}{d}(x-b)\right) +  \xi\left(2^\frac{k}{d}(x-b')\right)\right).
\end{align}
As in the proof of conditions $3.14,3.15$, there exists a constant $C'$ such that
\begin{equation}
\xi(x-b) + \xi(x-b') \le C'\frac{1}{(c^{-2}+\|x-b \|^d)^{2\sigma}}\frac{1}{(c^{-1}+\|x-b\|^d)^{1+\epsilon}}.
\end{equation}
We then get
\begin{align}
&\frac{|S_k(x,b)-S_k(x',b) - S_k(x,b') + S_k(x',b')|}{\rho(x,x')^\sigma \rho(b,b')^\sigma}\\
&\le 2^k2^\frac{2k}{d}\left(\xi\left(2^\frac{k}{d}(x-b)\right) + \xi\left(2^\frac{k}{d}(x-b')\right) \right)\\ 
& \le C'\frac{2^\frac{2k}{d}}{(c^{-2}+2^k\|x-b \|^d)^{2\sigma}}\frac{2^k}{(c^{-1}+2^k\|x-b\|^d)^{1+\epsilon}}\\
& = C'\frac{1}{(c^{-2}2^{-k}+\|x-b \|^d)^{2\sigma}}\frac{2^{-k\epsilon}}{(c^{-1}2^{-k}+\|x-b\|^d)^{1+\epsilon}}\\
& = C_3\frac{1}{(2^{-k}+\rho(x,b))^{2\sigma}}\frac{2^{-k\epsilon}}{(2^{-k}+\rho(x,b))^{1+\epsilon}},
\end{align}
where $C_3 = c^{2\sigma+1+\epsilon}$.
\end{proof}
\end{itemize}
Finally, we set $C = \max\{C_1,C_2,C_3\}$.
\end{proof}




\section{Proof of Lemma \ref{lemma:c2Approx}}
\label{app:c2Approx}
We first prove the following propositions.
\begin{proposition}
\label{prop:vanishing}
For each ${k,b}$, $\psi_{k,b},\, \widetilde{\psi}_{k,b}$  have two vanishing moments.
\end{proposition}

\begin{proof}
Note that a function $f$ on $\mathbb{R}^d$ which is symmetric about the origin satisfies
\begin{equation}
\int_{\mathbb{R}^d} xf(x)dx=0.
\end{equation}
We first show that for every $(k,b) \in \Lambda$, $\psi_{k,b}$ has two vanishing moments.
For each $(k,b) \in \mathbb{Z}\times\mathbb{R}^d$
\begin{align}
2^{-k}\int_{\mathbb{R}^d}\varphi(2^\frac{k}{d}(x-b))dx & = \int_{\mathbb{R}^d}\varphi(x)dx\\
& = 1,
\end{align}
by change of variables. This gives that for every $(k,b) \in \mathbb{Z}\times\mathbb{R}^d$
\begin{align}
\int_{\mathbb{R}^d}\psi_{k,b}(x)dx &= 2^{\frac{k}{2}}\int_{\mathbb{R}^d}\varphi(2^\frac{k}{d}(x-b) - \varphi\left(2^\frac{k-1}{d}(x-b)\right)dx \\
&=0,
\end{align}
Hence the first moment of $\psi_{k,b}$ vanishes. Further, since $\varphi$ is symmetric about the origin we have 
\begin{align}
\int_{\mathbb{R}^d}x\varphi\left(2^\frac{k}{d}(x-b)\right)dx & = \int_{\mathbb{R}^d}(2^{-\frac{k}{d}}y+b)\varphi(y)dy\\
& = 2^{-k}b\int_{\mathbb{R}^d}\varphi(y)dy\\
& = 2^{-k}b,
\end{align}
which gives
\begin{align}
\int_{\mathbb{R}^d}x\psi_{k,b}(x)dx &= 2^{-\frac{k}{2}}\int_{\mathbb{R}^d}\varphi\left(2^\frac{k}{d}(x-b) \right)-2^{-1}\varphi\left(2^\frac{k-1}{d}(x-b) \right)dx\\
&= 2^{-\frac{k}{2}}\left(2^{-k}b - 2^{-1}2^{-(k-1)}b\right) \\
&= 2^{-\frac{k}{2}}\left(2^{-k}b - 2^{-k}b\right) \\
&=0,
\end{align}
hence the second moment of $\psi_{k,b}$ also vanishes.

Finally, to show that the functions $\widetilde{\psi}_{k,b}$ have two vanishing moments as well,
we note that the dual functions are obtained using convolution with operators $D_k$ (\cite{deng2009harmonic}, p. 82), which, by the above arguments, have two vanishing moments; hence they inherit this property. 
\end{proof}

\begin{proposition} \label{prop:fastDecay}
For every $(k,b)$, $\widehat{\psi}_{k,b}$ decays faster than any polynomial.
\end{proposition}
\begin{proof}
By (\cite{deng2009harmonic}, p. 82), the dual functions are also wavelets, hence they satisfy condition $3.14$ in \cite{deng2009harmonic} with $\epsilon' < \epsilon$. Since in the proof of Lemma \ref{lemma:conditions}, $\epsilon$ can be arbitrarily large, it implies that the duals satisfy condition $3.14$ with any $\epsilon$, which proves the proposition.   
\end{proof}

\begin{proposition}\label{prop:boundedWavelet}
$|\psi_{k, b}| \le 2^{\frac{k}{2}-2}$.
\end{proposition}
\begin{proof}
We note that for all $d \ge2$, $C_d \le \frac{1}{2\cdot 2^d} \le \frac{1}{8}$. Hence $\varphi(x)\le \frac{1}{4}$, and consequently $|\psi(x)| \le \frac{1}{4}$. Since
\begin{equation}
\psi_{k,b}(x) =2^\frac{k}{2}\psi\left(2^\frac{k}{d}x-b)\right) 
\end{equation}
we get that $|\psi_{k, b}| \le 2^{\frac{k}{2}-2}$.
\end{proof}

\begin{proposition}\label{prop:coeffs}
if $f\in C^2$ and $\|\nabla^2_f \|_{op}$ is bounded, then
The coefficients $\langle \widetilde{\psi}_{k,b},f \rangle$ satisfy
\begin{equation}
|\langle\widetilde{\psi}_{k,b},f\rangle| = O(2^{-({2\frac{k}{d}}+\frac{k}{2} )})
\end{equation}
\end{proposition}
\begin{proof}
\begin{align}
\langle\widetilde{\psi}_{k,b},f\rangle  &=2^\frac{k}{2}\int_{\mathbb{R}^d }\widetilde{\psi}\left(2^\frac{k}{d}(x-b)\right)f(x)dx\\
&= 2^{-\frac{k}{2}}\int_{\supp(\widetilde{\psi})}\widetilde{\psi}(y)f(2^{-\frac{k}{d}} y + b)dy.
\end{align}
where we have used change of variables. 
Since that $f$ is twice differentiable, we can replace $f$ by its Taylor expansion near $b$ 
\begin{align}
&\int_{\supp(\widetilde{\psi})}\widetilde{\psi}(y)f(2^{-\frac{k}{d}} y + b)dy \\
=&\int_{\supp(\widetilde{\psi})}\widetilde{\psi}(y)\left(f(b) + 2^{-\frac{k}{d}} \langle y, \nabla_f(b) \rangle + O(\|\nabla^2_f(b) \|_{op}(2^{-\frac{k}{d}} \|y\|_2)^2)\right)dy. \label{eq:epsilon2}
\end{align}
By Proposition \ref{prop:vanishing} $\widetilde{\psi}$ has two vanishing moments; this gives
\begin{align}
|\langle\widetilde{\psi}_{k,b},f\rangle| & =  O\left(2^{-({2\frac{k}{d}}+\frac{k}{2} )}\|\nabla^2_f(b) \|_{op}\int_{\supp(\widetilde{\psi})}\widetilde{\psi}(y)\|y \|_2^2dy\right)
\end{align}
Since by Proposition~\ref{prop:fastDecay} $\widetilde{\psi}(y)$ decays exponentially fast, the integral $\int_{\supp(\widetilde{\psi})}\widetilde{\psi}(y)\|y \|_2^2dy$ is some finite number. As a result, 
\begin{equation}
|\langle\widetilde{\psi}_{k,b},f\rangle| = O(2^{-({2\frac{k}{d}}+\frac{k}{2} )}).
\end{equation}

\end{proof}

We will also use the following property:
\begin{remark}\label{remark:numberOfTerms}
 Every $x$ is in the support of at most $12^d$ wavelet terms at every scale. 
\end{remark}

We are now ready to prove Lemma \ref{lemma:c2Approx}
\begin{proof}
Let $f\in L_2(\mathbb{R}^d)$, $d \le 3$ be compactly supported, twice differentiable and with $\|\nabla^2_f\|_{op}$ bounded. $f$ can be expressed as 
\begin{equation}
f = \sum_{(k,b) \in \Lambda }\langle \widetilde{\psi}_{k,b}, f\rangle \psi_{k,b}. 
\end{equation}
approximating $f$ by $f_K$, which  only consists of the wavelet terms of scales $k \le K$, we obtain that for every $x \in \mathbb{R}^d$
\begin{equation}
|f(x) - f_K(x)| \le \sum_{k=K+1}^\infty \sum_{b \in 2^{-k}\mathbb{Z}}|\psi_{k,b} |\langle \widetilde{\psi}_{k,b}, f\rangle. \label{eq:111}
\end{equation}
By Remark \ref{remark:numberOfTerms}, at most $12^d$ wavelet terms are supported on $x$ at every scale; by Proposition \ref{prop:boundedWavelet} $|\psi_{k, b}| \le 2^{\frac{k}{2}-2}$; by Proposition \ref{prop:coeffs} $|\langle\widetilde{\psi}_{k,b},f\rangle| = O(2^{-(\frac{2k}{d} + \frac{k}{2})})$. Plugging these into Equation \eqref{eq:111} gives
\begin{align}
|f(x) - f_K(x)| &= O \left(\sum_{k=K+1}^\infty 12^d 2^{\frac{k}{2}-2} 2^{-(\frac{2k}{d} + \frac{k}{2})}\right) \\
&=O \left(\sum_{k=K+1}^\infty 2^{-\frac{2k}{d}}\right)\\
&=O \left(2^{-\frac{2K}{d}}\right).
\end{align}
\end{proof}

 \end{document}